\documentclass{article}
\usepackage{arxiv}

\usepackage[utf8]{inputenc}
\usepackage[T1]{fontenc}
\usepackage[english]{babel}
\usepackage{lineno}
\usepackage{csquotes}
\usepackage{microtype}
\usepackage{pifont}

\usepackage[parfill]{parskip}
\usepackage{enumitem}
\usepackage{framed}
\usepackage{xspace}
\usepackage{placeins}
\usepackage{wrapfig}
\usepackage[usenames,dvipsnames]{xcolor}
\definecolor{shadecolor}{gray}{0.9}
\colorlet{DarkBlue}{black!50!blue!100}
\colorlet{DarkRed}{black!15!red!100}

\usepackage{graphicx}
\usepackage{subcaption}
\usepackage{tikz}
\usepackage{pgf}
\usepackage{pgfplots}
\usepgfplotslibrary{patchplots}
\pgfplotsset{compat=1.3}

\usetikzlibrary{bayesnet}
\pgfdeclarelayer{edgelayer}
\pgfdeclarelayer{nodelayer}
\pgfsetlayers{edgelayer,nodelayer,main}
\definecolor{hexcolor0xbfbfbf}{rgb}{0.749,0.749,0.749}
\usetikzlibrary{arrows,arrows.meta,automata,backgrounds}
\usetikzlibrary{decorations.pathreplacing}
\usetikzlibrary{decorations.markings}
\tikzset{>=latex}
\tikzstyle{none}   = [inner sep=0pt]
\tikzstyle{line}   = [ -, thick, shorten <=1pt, shorten >=1pt ]
\tikzstyle{arrow}  = [ ->, thick, shorten <=1pt, shorten >=1pt ]
\tikzstyle{ardash} = [ dashed, ->, thick, shorten <=1pt, shorten >=1pt ]
\tikzstyle{empty}=[circle,opacity=0.0,text opacity=1.0,inner sep=0pt]
\tikzstyle{box}=[rectangle,fill=White,draw=Black]
\tikzstyle{filled}=[circle,thick,fill=hexcolor0xbfbfbf,draw=Black]
\tikzstyle{hollow}=[circle,thick,fill=White,draw=Black]
\tikzstyle{param}=[rectangle,fill=Black,draw=Black,inner sep=0pt,minimum width=4pt,minimum height=4pt]
\tikzstyle{paramhollow}=[rectangle,thick,fill=White,draw=Black,inner sep=0pt,minimum
width=4pt,minimum height=4pt]
\usetikzlibrary{calc}
\tikzset{
  myarrow/.style={stealth-,shorten >=3pt,shorten <=3pt}
}
\tikzset{
testarrow/.style={draw,decoration={markings, mark=at position 1 with {\arrow{>}}},postaction={decorate}}
}

\usepackage{booktabs, array}

\usepackage[
  linktoc=all,
  pdfcenterwindow=true,
  bookmarks=true,
  colorlinks,
  citecolor=DarkBlue,
  linkcolor=DarkBlue,
  urlcolor=DarkBlue]{hyperref}

\usepackage[all]{hypcap}

\usepackage[noline]{algorithm2e}
\usepackage{listings}
\usepackage{fancyvrb}
\fvset{fontsize=\small}

\usepackage{setspace}
\usepackage[cmex10]{amsmath}
\usepackage{mathtools,amssymb,amsthm}
\usepackage{nicefrac}

\newtheorem{thm}{Theorem}
\newtheorem{defn}{Definition}
\newtheorem{lem}{Lemma}

\newtheorem{cor}{Corollary}

\DeclarePairedDelimiterX{\inp}[2]{\langle}{\rangle}{#1, #2} % < , >
\DeclareMathOperator*{\argmin}{arg\,min  \ }

\DeclareMathOperator{\rR}{{\mathbb{R}}}

\DeclareMathOperator{\cU}{\mathcal{U}}

\DeclareMathOperator{\cW}{\mathcal{W}}
\DeclareMathOperator{\cX}{\mathcal{X}}

\usepackage{cleveref}

\crefname{equation}{Eq.}{Eqs.}
\creflabelformat{equation}{#2\textup{#1}#3}
\crefname{lem}{lemma}{lemmas}
\crefname{thm}{theorem}{theorems}
\crefname{ass}{assumption}{assumptions}
\crefname{defn}{definition}{definitions}
\crefname{cor}{corollary}{corollaries}

\RestyleAlgo{ruled}
\usepackage{tabularx}

\title{Optimistic Meta-Gradients}

\author{%
  Sebastian Flennerhag\\
  DeepMind\\
  \texttt{flennerhag@google.com}\\
  \And
  Tom Zahavy\\
  DeepMind\\
  \\
  \And
  Brendan O'Donoghue\\
  DeepMind\\
  \\
  \And
  Hado van Hasselt\\
  DeepMind\\
  \And
  Andr{\'a}s Gy{\"o}rgy\\
  DeepMind\\
  \And
  Satinder Singh\\
  DeepMind\\
}

\begin{document}

\maketitle

\begin{abstract}
We study the connection between gradient-based meta-learning and convex optimisation. We observe that gradient descent with momentum is a special case of meta-gradients, and building on recent results in optimisation, we prove convergence rates for meta-learning in the single task setting. While a meta-learned update rule can yield faster convergence up to constant factor, it is not sufficient for acceleration. Instead, some form of optimism is required. We show that optimism in meta-learning can be captured through Bootstrapped Meta-Gradients \citep{flennerhag2022bootstrapped}, providing deeper insight into its underlying mechanics.
\end{abstract}

\section{Introduction}

\begin{wrapfigure}{r}{0.5\linewidth}
    \centering
    \vspace*{-18pt}
    \includegraphics[width=.97\linewidth]{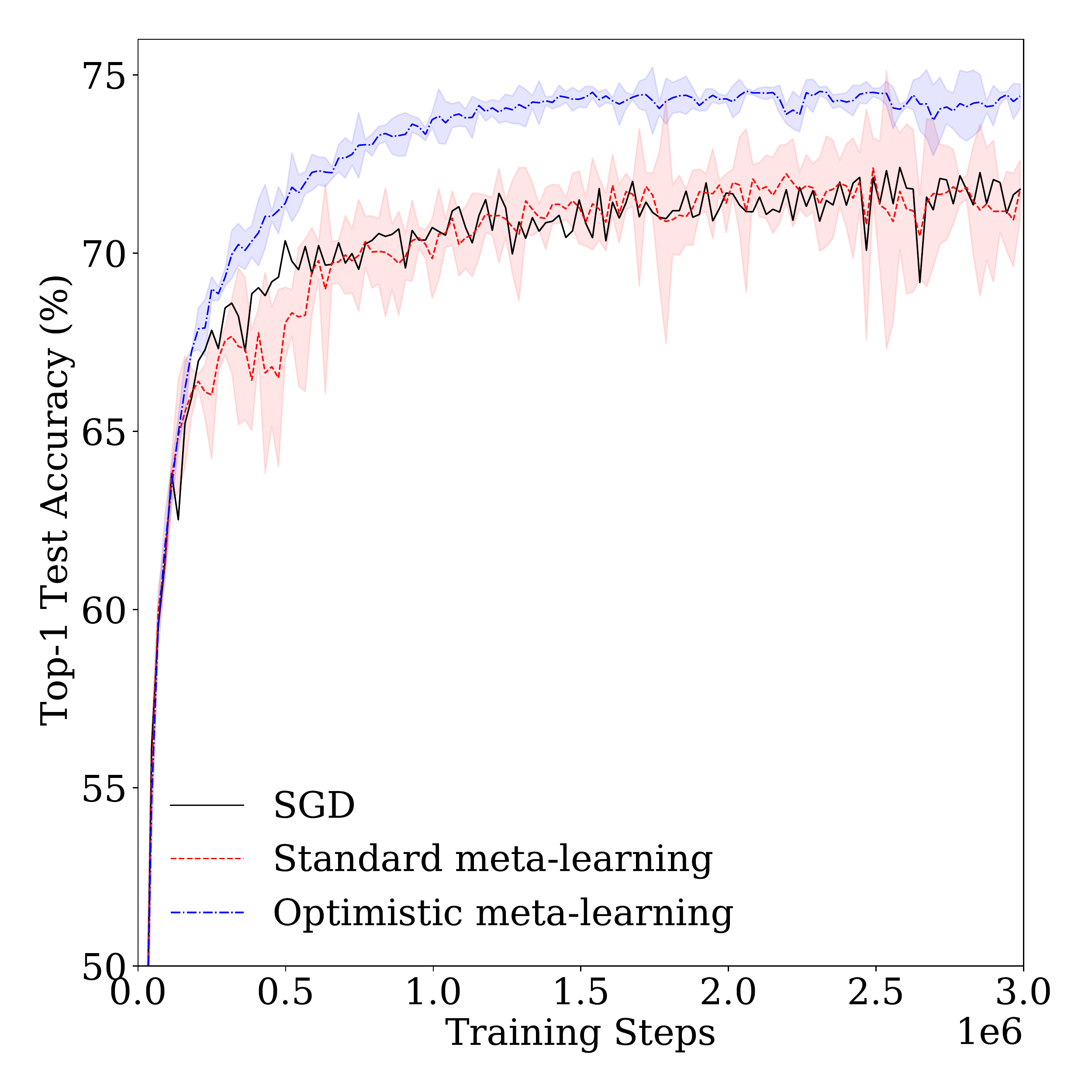}
    \caption{ImageNet. We compare training a 50-layer ResNet using SGD against variants that tune an element-wise learning rate online using standard meta-learning or optimistic meta-learning. Shading depicts 95\% confidence intervals over 3 seeds.}
    \label{fig:imagenet}
    \vspace*{-14pt}
\end{wrapfigure}

In meta-learning, a learner is using a parameterised algorithm to adapt to a given task. The parameters of the algorithm are then meta-learned by evaluating the learner's resulting performance \citep{Schmidhuber:1987ev,Hinton:1987fa,Bengio:1991le}. This paradigm has garnered wide empirical success \citep{Hospedales:2020survey}. For instance, it has been used to meta-learn how to explore in reinforcement learning (RL) \citep{Xu:2018explore,Alet:2020me}, online hyper-parameter tuning of non-convex loss functions \citep{Bengio:2000hypergradients,Maclaurin:2015hypergradients,Xu:2018metagradient,Zahavy:2020se}, discovering black-box loss functions \citep{Chen2016:le,Kirsch:2019im,xu2020meta,Oh:2020discoveringrl}, black-box learning algorithms \citep{Hochreiter:2001le,Wang:2016re}, or entire training protocols \citep{Real:2020ze}. Yet, very little is known in terms of the theoretical properties of meta-learning. 

The reason for this is the complex interaction between the learner and the meta-learner. \textbf{learner's problem} is to minimize the expected loss $f$ of a stochastic objective by adapting its parameters $x \in \rR^{n}$. The learner has an update rule $\varphi$ at its disposal that generates new parameters $x_{t} = x_{t-1} + \varphi(x_{t-1}, w_t)$; we suppress data dependence to simplify notation. A simple example is when $\varphi$ represents gradient descent with $w_t = \eta$ its step size, that is $\varphi(x_{t-1}, \eta) = -\eta \nabla f(x_{t-1})$ \citep{Mahmood2012tuning,vanErven:2016metagrad}; several works have explored meta-learning other aspects of a gradient-based update rule \citep{Finn:2017maml,Nichol:2018uo,Flennerhag:2019tl,Xu:2018metagradient,Zahavy:2020se,flennerhag2022bootstrapped,Kirsch:2019im,Oh:2020discoveringrl}. More generally, $\varphi$ need not be limited to the gradient of any function, for instance, it can represent some algorithm implemented within a Recurrent Neural Network \citep{Schmidhuber:1987ev,Hochreiter:2001le,Andrychowicz:2016tf,Wang:2016re}. 

\textbf{The meta-learner's problem} is to optimise the meta-parameters $w_t$ to yield effective \emph{updates}. In a typical (gradient-based) meta-learning setting, it does so by treating $x_{t}$ as a function of $w$. Let $h_t$, defined by $h_t(w) = f(x_{t-1} + \varphi(x_{t-1}, w))$, denote the learner's post-update performance as a function of $w$. The learner and the meta-learner co-evolve according to
\begin{align*}
x_{t} = x_{t-1} + \varphi(x_{t-1}, w_t), \qquad  \text{and} \qquad w_{t+1} &= w_t - \nabla h_t(w_t)\\
&= w_t - D \varphi(x_{t-1}, w_t)^T \nabla f(x_t),
\end{align*}
where $D \varphi(x, w)$ denotes the Jacobian of $\varphi$ with respect to $w$. The nested structure between these two updates makes it challenging to analyse meta-learning, in particular it depends heavily on the properties of the Jacobian. In practice, $\varphi$ is highly complex and so $D \varphi$ is almost always intractable. For instance, in \citet{Xu:2018explore}, the meta-parameters define the data-distribution under which a stochastic gradient is computed. In \citet{Zahavy:2020se}, the meta-parameters define auxiliary objectives that are meant to help with representation learning; in \citet{Vinyals:2016uj} they learn an embedding space for nearest-neighbour predictions.

For this reason, the only theoretical results we are aware of specialise to the multi-task setting and assume $\varphi$ represents adaptation by gradient descent. In this setting, at each iteration $t$, the learner must adapt to a new task $f_t$. The learner adapts by taking a (or several) gradient step(s) on $f_t$ using either a meta-learned initialisation \citep{Flennerhag:2019tl,finn2019online,Fallah:2020maml,Wang:2022maml} or using a meta-learned regulariser \citep{khodak2019adaptive,Denevi:2019online}. Because the update rule has this form, it is possible to treat the meta-optimisation problem as an online learning problem and derive convergence guarantees. Acceleration in this setup is driven by the tasks similarity. That is, if all tasks are sufficiently similar, a meta-learned update can accelerate convergence \citep{khodak2019adaptive}. However, these results do not yield acceleration in the absence of a task distribution to the best of our knowledge.

This paper provides an alternative view. We study the classical convex optimisation setting of approximating the minimiser $\min_{x} f(x)$. We observe that setting the update rule equal to the gradient, i.e. $\varphi: (x, w) \mapsto w \nabla f(x)$, recovers gradient descent. Similarly, we show in \Cref{sec:res} that $\varphi$ can be chosen to recover gradient descent with momentum. This offers another view of meta-learning as a non-linear transformation of classical optimisation. A direct implication of this is that a task similarity is not necessary condition for improving the rate of convergence via meta-learning. While there is ample empirical evidence to that effect \citep{Xu:2018metagradient,Zahavy:2020se,flennerhag2022bootstrapped, Luketina:2022metagradients}, we are only aware of theoretical results in the special case of meta-learned step sizes \citep{Mahmood2012tuning,vanErven:2016metagrad}. 

In particular, we analyse meta-learning using recent techniques developed for convex optimisation \citep{Cutkosky:2019anytime,Joulani:2020simpler,Wang:2021FenchelGame}. Given a function $f$ that is convex with Lipschitz smooth gradients, meta-learning improves the rate of convergence by a multiplicative factor $\lambda$ to $O(\lambda/T)$, via the smoothness of the update rule. Importantly, these works show that to achieve accelerated convergence, $O(1/T^2)$, some form of optimism is required. This optimism essentially provides a prediction of the next gradient, and hence represents a model of the geometry. We consider optimism with meta-learning in the convex setting and prove accelerated rates of convergence, $O(\lambda / T^2)$. Again, meta-learning affects these bounds by a multiplicative factor. We further show that optimism in meta-learning can be expressed through the recently proposed Bootstrapped Meta-Gradient method \citep[BMG;][]{flennerhag2022bootstrapped}. Our analysis provides a first proof of convergence for BMG and highlights the underlying mechanics that enable faster learning with BMG. Our main contributions are as follows: 
\begin{enumerate}[noitemsep]
  \item We show that meta-learning contains gradient descent with momentum (Heavy Ball \citep{Polyak:1964heavyball}; \Cref{sec:res}) and Nesterov Acceleration \citep{nesterov1983method} as special cases (\Cref{sec:bmg}). 
  \item We show that gradient-based meta-learning can be understood as a non-linear transformation of an underlying optimisation method (\Cref{sec:res}).
  \item We establish rates of convergence for meta-learning in the convex setting (\Cref{sec:main,sec:bmg}).
  \item We show that optimism can be expressed through \citep{flennerhag2022bootstrapped}. Our analysis (\Cref{sec:bmg}) provides a first proof of convergence for BMG.
\end{enumerate}

\begin{minipage}[t]{0.44\linewidth}
\begin{algorithm}[H]
\DontPrintSemicolon
\Indp
\SetKwInOut{Input}{input}
\SetKwInOut{Empty}{}
\SetKwFor{For}{for }{}{}
\Input{Weights $\{\beta_{t}\}^T_{t=1}$}
\Input{Update rule $\varphi$}
\Input{Initialisation $(x_0, w_1)$}
\For{$t = 1, 2, \ldots, T$:}{
$x_{t} = x_{t-1} + \varphi(x_{t-1}, w_t)$ \\
$h_t(\cdot) = f(x_{t-1} +  \rho_t\varphi(x_{t-1}, \cdot))$\\
$w_{t+1} = w_t - \beta_t \nabla h_t(w_t)$
}
\Return{$x_T$}
\;
\;
\caption{Meta-learning in practice.}\label{alg:mg}
\end{algorithm}
\end{minipage}
\hfill
\begin{minipage}[t]{0.55\textwidth}
\begin{algorithm}[H]
\DontPrintSemicolon
\Indp
\SetKwInOut{Input}{input}
\SetKwFor{For}{for }{}{}
\Input{Weights $\{\alpha_{t}\}^T_{t=1}, \{\beta_{t}\}^T_{t=1}$}
\Input{Update rule $\varphi$}
\Input{Initialisation $(\bar{x}_0, w_1)$}
\For{$t = 1, 2, \ldots, T$:}{
$x_t = \varphi(\bar{x}_{t-1}, w_t)$\\
$\bar{x}_t = (1-\alpha_t/\alpha_{1:t})\bar{x}_{t-1} + (\alpha_t/\alpha_{1:t}) x_t$\\
$g_t = D \varphi(\bar{x}_{t-1}, w_t)^T \nabla f(\bar{x}_t)$ \\
$w_{t+1} \!\! = \! \arg\min_{w \in \cW}\! \sum_{s=1}^t \! \alpha_s \inp{g_s}{w} \! +\!  \frac{1}{2\beta_t}\| w \|^2\!\!\!\!\!\!\!\!\!\!$
}
\Return{$\bar{x}_T$}
\caption{Meta-learning in the convex setting.}\label{alg:model}
\end{algorithm}
\end{minipage}

\section{Meta-learning meets convex optimisation}\label{sec:setup}

\paragraph{Problem definition.} This section defines the problem studied in this paper and introduces our notation. Let $f: \cX \to \rR$ be a proper and convex function. The problem of interest is to approximate the global minimum $\min_{x \in \cX} f(x)$. We assume a global minimiser exists and is unique, defined by
\begin{equation}\label{eq:problem}
x^* = \argmin_{x \in \cX} f(x).
\end{equation}

We assume that $\cX \subseteq \rR^n$ is a closed, convex and non-empty set. $f$ is differentiable and has Lipschitz smooth gradients with respect to a norm $\| \cdot \|$, meaning that there exists $L \in (0, \infty)$ such that $\| \nabla f(x) - \nabla f(y) \|_{*} \leq L \| x - y \|$ for all $x, y \in \cX$, where $\| \cdot \|_{*}$ is the dual norm of $\| \cdot \|$. We consider the noiseless setting for simplicity; our results carry over to the stochastic setting by replacing the key online-to-batch bound used in our analysis by its stochastic counterpart \citep{Joulani:2020simpler}.

\paragraph{Algorithm.} \Cref{alg:mg} describes a typical meta-learning algorithm. Unfortunately, at this level of generality, little can be said about the its convergence properties. Instead, we consider a stylized variant of meta-learning, described in \Cref{alg:model}. This model differs in three regards: (a) it relies on moving averages (b) we use a different online learning algorithm for the meta-update, and (c) we make stricter assumptions on the update rule. We describe each component in turn.

Let $[T] = \{1, 2, \ldots, T\}$. We are given weights $\{\alpha_{t}\}_{t=1}^T$, each $\alpha_t > 0$, and an initialisation $(\bar{x}_0, w_1) \in \cX \times \cW$. At each time $t \in [T]$, an update rule $\varphi: \cX \times \cW \to \cX$ generates the update $x_t = \varphi(\bar{x}_{t-1}, w_t)$, where $\cW \subseteq \rR^m$ is closed, convex, and non-empty. We discuss $\varphi$ momentarily. The algorithm maintains the online average
\begin{equation}\label{eq:avg}
\bar{x}_t = \frac{x_{1:t}}{\alpha_{1:t}} = (1-\rho_t) \bar{x}_{t-1} + \rho_t x_t,
\end{equation}
where $x_{1:t} = \sum_{s=1}^t \alpha_s x_s$, $\alpha_{1:t} = \sum_{s=1}^t \alpha_s$, and $\rho_t = \alpha_t / \alpha_{1:t}$. Our goal is to establish conditions under which $\{\bar{x}_t\}_{t=1}^T$ converges to the minimiser $x^*$. While this moving average is not always used in practical applications, it is required for accelerated rates in online-to-batch conversion \citep{Wang:2018acc,Cutkosky:2019anytime,Joulani:2020simpler}. 

Convergence depends on how each $w_t$ is chosen. In \Cref{alg:mg}, the meta-learner faces a sequence of losses $h_t: \cW \to \rR$ defined by the composition $h_t(w) = f((1-\rho_t) \bar{x}_{t-1} + \rho_t \varphi(\bar{x}_{t-1}, w))$. This makes meta-learning a form of online optimisation \citep{Mcmahan2017survey}. The meta-updates in \Cref{alg:mg} is an instance of online gradient descent, which we can model as Follow-The-Regularized-Leader (FTRL; reviewed in \Cref{sec:bg}). Given some norm $\| \cdot \|$,  an initialization $w_0$ and $\beta > 0$, FTRL sets each $w_t$ according to
\begin{equation}\label{eq:mg}
w_{t+1} = \argmin_{w \in \cW} \left(\sum_{s=1}^t \alpha_s \inp{\nabla h_s(w_s)}{w} + \frac{1}{2 \beta} \| w \|^2 \right).
\end{equation}
If $\| \cdot \|$ is the Euclidean norm, the interior solution to \Cref{eq:mg} is given by $w_{t+1} = w_t - \alpha_t  \beta \nabla h_t (w_t)$, the meta-update in \Cref{alg:mg}. It is straightforward to extend \Cref{eq:mg} to account for meta-updates that use AdaGrad-like \citep{Duchi:2011adagrad} acceleration by altering the norms \citep{Joulani:2017modular}.

\paragraph{Update rule.} It is not possible to prove convergence outside of the convex setting, since $\varphi$ may reach a local minimum where it cannot yield better updates, but the updates are not sufficient to converge. Convexity means that each $h_t$ must be convex, which requires that $\varphi$ is affine in $w$ (but may vary non-linearly in $x$). We also assume that $\varphi$ is smooth with respect to $\| \cdot \|$, in the sense that it has bounded norm; for all $x \in \cX$ and all $w \in \cW$ we assume that there exists $\lambda \in (0, \infty)$ for which
\begin{equation*}
\| D \varphi(x, w)^T \nabla f(x) \|^2_{*} \leq \lambda \| \nabla f(x) \|^2_{*}.
\end{equation*}
These assumptions hold for any smooth update rule up to first-order Taylor approximation error. 

\begin{figure}
    \centering
    \includegraphics[width=\linewidth]{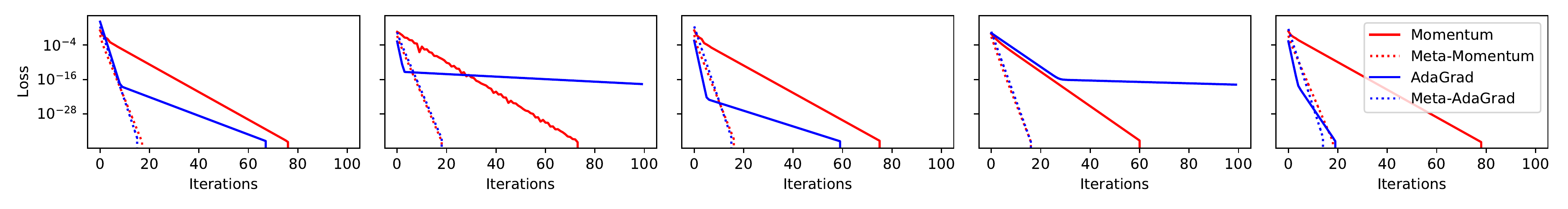}
    \includegraphics[width=.97\linewidth]{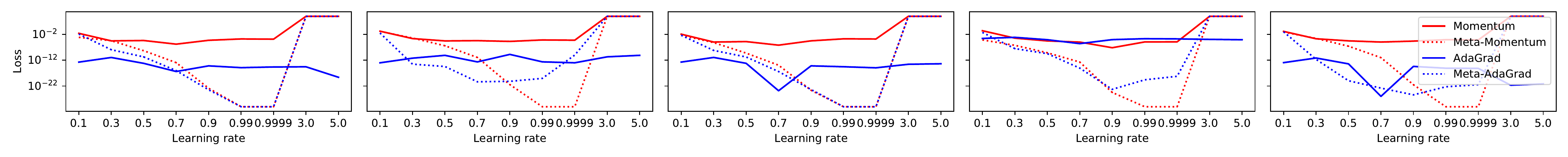}
    \caption{Convex Quadratic. We generate convex quadratic loss functions with ill-conditioning and compare gradient descent with momentum and AdaGrad to meta-learning variants. Meta-Momentum uses $\varphi: (x, w) \mapsto w \odot \nabla f(x)$ while Meta-AdaGrad uses $\varphi: (x, w) \mapsto \nabla f(x) / \sqrt{w}$, where division is element-wise. \emph{Top:} loss per iteration for randomly sampled loss functions. \emph{Bottom:} cumulative loss (regret) at the end of learning as a function of learning rate; details in \Cref{app:experiments}.}
    \label{fig:conv-2d}
\end{figure}

\section{Meta-Gradients in the Convex Setting - An Overview}\label{sec:res}

In this section, we provide an informal discussion of our main results (full analysis; \Cref{sec:main,sec:bmg}). 

\paragraph{Meta-Gradients without Optimism.} The main difference between classical optimisation and meta-learning is the introduction of the update rule $\varphi$. To see how this acts on optimisation, consider two special cases. If the update rule just return the gradient, $\varphi = \nabla f$, \Cref{alg:model} is reduced to gradient descent (with averaging). The inductive bias is fixed and does not change with past experience, and so acceleration is not possible---the rate of convergence is $O(1/\sqrt{T})$ \citep{Wang:2021FenchelGame}. The other extreme is an update rule that only depends on the meta-parameters, $\varphi(x, w) = w$. Here, the meta-learner has ultimate control and selects the next update without constraints. The only relevant inductive bias is contained in $w$. To see how this inductive bias is formed, suppose $\| \cdot \| = \| \cdot \|_2$ so that \Cref{eq:mg} yields $w_{t+1} = w_t - \alpha_t \rho_t \beta \nabla f(\bar{x}_t)$ (assuming an interior solution). Combining this with the moving average in \Cref{eq:avg}, we may write the learner's iterates as
\begin{equation*}
\bar{x}_{t} = \bar{x}_{t-1} + \tilde{\rho}_t \left( \bar{x}_{t-1} - \bar{x}_{t-2} \right) - \tilde{\beta}_t \nabla f(\bar{x}_{t-1}),
\end{equation*}
where each $\tilde{\rho_t} = \rho_t \frac{1-\rho_{t-1}}{\rho_{t-1}}$ and $\tilde{\beta}_{t} = \alpha_t \rho_t \beta$; setting $\beta = 1/(2L)$ and each $\alpha_t = t$ yields $\tilde{\rho}_t = \frac{t-2}{t+1}$ and $\tilde{\beta}_t = t/(4(t+1)L)$. Hence, the canonical momentum algorithm, Polyak's Heavy-Ball method \citep{Polyak:1964heavyball}, is obtained as the special case of meta-learning under the update rule $\varphi: (x, w) \mapsto w$. Because Heavy Ball carries momentum from past updates, it can encode a model of the learning dynamics that leads to faster convergence, on the order $O(1/T)$. The implication of this is that the dynamics of meta-learning are fundamentally momentum-based and thus learns an update rule in the same cumulative manner. This manifests theoretically through its convergence guarantees.

\begin{thm}[\emph{Informal}]\label{thm:mg-informal}
Set $\alpha_t = 1$ and $\beta = \frac{1}{\lambda L}$. If each $x_t$ is generated under \Cref{alg:model}, then for any viable $\varphi$, $f(\bar{x}_T) - f(x^*) \leq \frac{\lambda L \operatorname{diam}(\cW)}{T}$.
\end{thm}

We refer the reader to \Cref{thm:mg} for a formal statement.
Compared to Heavy Ball, meta-learning introduces a constant $\lambda$ that captures the smoothness of the update rule. Hence, while meta-learning does not achieve better scaling in $T$ through $\varphi$, it can improve upon classical optimisation by a constant factor if $\lambda < 1$. That meta-learning can improve upon momentum is borne out experimentally. In \Cref{fig:conv-2d}, we consider the problem of minimizing a convex quadratic $f: x \mapsto \inp{x}{Q x}$, where $Q \in \rR^{n \times n}$ is PSD but ill-conditioned. We compare momentum to a meta-learned step-size, i.e. $\varphi: (x, w) \mapsto w \odot \nabla f(x)$, where $\odot$ is the Hadamard product. Across randomly sampled $Q$ matrices (details: \Cref{app:experiments}), we find that introducing a non-linearity $\varphi$ leads to a sizeable improvement in the rate of convergence. We also compare AdaGrad to a meta-learned version, $\varphi: (x, w) \mapsto \nabla f(x) / \sqrt{w}$, where division is element-wise. While AdaGrad is a stronger baseline on account of being parameter-free, we find that meta-learning the scale vector consistently leads to faster convergence.

\paragraph{Meta-Gradients with Optimism.} It is well known that minimizing a smooth convex function admits convergence rates of $O(1/T^2)$. Our analysis of standard meta-gradients does not achieve such acceleration. Previous work indicate that we should not expect to either; to achieve the theoretical lower-limit of $O(1/T^2)$, some form of \emph{optimism} (reviewed in \Cref{sec:bg}) is required. A typical form of optimism is to predict the next gradient. This is how Nesterov Acceleration operates \citep{nesterov1983method} and is the reason for its $O(1/T^2)$ convergence guarantee. 

From our perspective, meta-learning is a non-linear transformation of the iterate $x$. Hence, we should expect optimism to play a similarly crucial role. Formally, optimism comes in the form of \emph{hint functions} $\{ \tilde{g}_{t} \}_{t=1}^T$, each $\tilde{g}_t \in \rR^m$, that are revealed to the meta-learner prior to selecting $w_{t+1}$. These hints give rise to \emph{Optimistic Meta-Learning} (OML) via meta-updates
\begin{equation}\label{eq2:omg}
w_{t+1} = \argmin_{w \in \cW} \left(\alpha_{t+1} \tilde{g}_{t+1} + \sum_{s=1}^t \alpha_s \inp{\nabla h_s(w_s)}{w} + \frac{1}{2 \beta_t} \| w \|^2 \right).
\end{equation}

If the hints are accurate, meta-learning with optimism can achieve an accelerated rate of $O(\tilde{\lambda}/T^2)$, where $\tilde{\lambda}$ is a constant that characterises the smoothness of $\varphi$, akin to $\lambda$. Again, we find that meta-learning behaves as a non-linear transformation of classical optimism and its rate of convergence is governed by the geometry it induces. We summarise this result in the following result.

\begin{thm}[Informal]\label{thm:omg-main}
Let each hint be given by $\tilde{g}_{t+1} = D \varphi(\bar{x}_{t-1}, w_t)^T \nabla f(\bar{x}_{t})$. Assume that $\varphi$ is sufficiently smooth. Set $\alpha_t = t$ and $\beta_t = \frac{t-1}{2t \tilde{\lambda} L}$, then $f(\bar{x}_T) - f(x^*) \leq \frac{4 \tilde{\lambda} L \operatorname{diam}(\cW)}{T^2 - 1}$.
\end{thm}

For a formal statement, see \Cref{thm:omg}. These predictions hold empirically in a non-convex setting. We train a 50-layer ResNet using either SGD with a fixed learning rate, or an update rule that adapts a per-parameter learning rate online, $\varphi: (x, w) \mapsto w \odot \nabla f(x)$. We compare the standard meta-learning approach without optimism to optimistic meta-learning. \Cref{fig:imagenet} shows that optimism is critical for meta-learning to achieve acceleration, as predicted by theory (experiment details in \Cref{app:imagenet}). 

\section{Analysis preliminaries: Online Convex Optimisation}\label{sec:bg}

In this section, we present analytical tools from the optimisation literature that we build upon. In a standard optimisation setting, there is no update rule $\varphi$; instead, the iterates $x_t$ are generated by a gradient-based algorithm, akin to \Cref{eq:mg}. In particular, our setting reduces to standard optimisation if $\varphi$ is defined by $\varphi: (x, w) \mapsto w$, in which case $x_t = w_t$. A common approach to analysis is to treat the iterates $x_1, x_2, \ldots$ as generated by an online learning algorithm over online losses, obtain a regret guarantee for the sequence, and use online-to-batch conversion to obtain a rate of convergence. 

\paragraph{Online Optimisation.} In online convex optimisation \citep{Zinkevich:2003online}, a learner is given a convex decision set $\cU$ and faces a sequence of convex loss functions $\{\alpha_t f_t\}_{t=1}^T$. At each time $t \in [T]$, it must make a prediction $u_t$ prior to observing $\alpha_t f_t$, after which it incurs a loss $\alpha_t f_t(u_t)$ and receives a signal---either $\alpha_t f_t$ itself or a (sub-)gradient of $\alpha_t f_t(u_t)$. The learner's goal is to minimise \emph{regret}, $R(T) \coloneqq \sum_{t=1}^T \alpha_t (f_t(u_t) - f_{t}(u))$, against a comparator $u \in \cU$. An important property of a convex function $f$ is $f(u') - f(u) \leq \inp{\nabla f(u')}{u' - u}$. Hence, the regret is largest under linear losses: $\sum_{t=1}^T \alpha_t (f_t(u_t) - f_{t}(u)) \leq \sum_{t=1}^T \alpha_t \inp{\nabla f_t(u_t)}{u_t - u}$. For this reason, it is sufficient to consider regret under linear loss functions. An algorithm has sublinear regret if $\lim_{T \to \infty} R(T) / T  = 0$.

\paragraph{FTRL \& AO-FTRL.} The meta-update in \Cref{eq:mg} is an instance of Follow-The-Regularised-Leader (FTRL) under linear losses. In \Cref{sec:bmg}, we show that BMG is an instance of the Adaptive-Optimistic FTRL (AO-FTRL), which is an extension due to \citep{Rakhlin:2013optimism,Mohri:2016aoftrl,Joulani:2020simpler,Wang:2021FenchelGame}. In AO-FTRL, we have a strongly convex regulariser $\| \cdot \|^2$. FTRL and AO-FTRL sets the first prediction $u_1$ to minimise $\| \cdot \|^2$. Given linear losses $\{g_s\}_{s=1}^{t-1}$ and learning rates $\{\beta_t\}_{t=1}^T$, each $\beta_{t} > 0$, the algorithm proceeds according to
\begin{equation}\label{eq:ftrl}
u_{t} = \argmin_{u \in \cU} \left( \alpha_t \inp{\tilde{g}_t}{u} + \sum_{s=1}^{t-1} \alpha_s \inp{g_s}{u} + \frac{1}{2 \beta_t} \| u \|^2 \right),
\end{equation}
where each $\tilde{g}_t$ is a ``hint'' that enables optimistic learning \citep{Rakhlin:2013optimism,Mohri:2016aoftrl}; setting $\tilde{g}_t =0$ recovers the original FTRL algorithm. The goal of a hint is to predict the next loss vector $g_t$; if the predictions are accurate AO-FTRL can achieve lower regret than its non-optimistic counter-part. Since $\| \cdot \|^2$ is strongly convex, FTRL is well defined in the sense that the minimiser exists, is unique and finite \citep{Mcmahan2017survey}. The regret of FTRL and AO-FTRL against any comparator $u \in \cU$ can be upper-bounded by
\begin{equation}\label{eq:ftrl-regret}
R(T) = \sum_{t=1}^T \alpha_t \inp{g_t}{u_t - u} \leq \frac{\|u\|^2}{2 \beta_T} + \frac12 \sum_{t=1}^T \alpha_t^2 \beta_t \left\| g_t - \tilde{g}_t \right\|_{*}^2.
\end{equation}
Hence, hints that predict $g_t$ well can reduce the regret substantially. Without hints, FTRL can guarantee $O(\sqrt{T})$ regret (for non strongly convex loss functions). However, \citet{Dekel:2017hints} show that under linear losses, if hints are weakly positively correlated---defined as $\inp{g_t}{\tilde{g}_t} \geq \epsilon \| g_t \|^2$ for some $\epsilon > 0$---then the regret guarantee improves to $O(\log T)$, even for non strongly-convex loss functions. We believe optimism provides an exciting opportunity for novel forms of meta-learning. Finally, we note that these regret bounds (and hence our analysis) can be extended to stochastic optimisation \citep{Mohri:2016aoftrl,Joulani:2017modular}.

\paragraph{Online-to-batch conversion.} The main idea behind online to batch conversion is that, for $f$ convex, Jensen's inequality gives $f(\bar{x}_T) - f(x^*) \leq \sum_{t=1}^T \alpha_t \inp{\nabla f(x_t)}{x_t - x^*} / \alpha_{1:T}$. Hence, one can provide a convergence rate by first establishing the regret of the algorithm that generates $x_t$, from which one obtains the convergence rate of the moving average of iterates. Applying this naively yields $O(1/T)$ rate of convergence. In recent work, \citet{Cutkosky:2019anytime} shows that one can upper-bound the sub-optimality gap by instead querying the gradient gradient at the average iterate, $f(\bar{x}_T) - f(x^*) \leq \sum_{t=1}^T \alpha_t \inp{\nabla f(\bar{x}_t)}{x_t - x^*} / \alpha_{1:T}$, which can yield faster rates of convergence. Recently, \citet{Joulani:2020simpler} tightened the analysis and proved that the sub-optimality gap can be bounded by
%
%\begin{equation*}
%f(\bar{x}_T) - f(x^*) \leq \frac{\sum_{t=1}^T \alpha_t \inp{\nabla f(\bar{x}_t)}{x_t - x^*} - B_{1:T} - \bar{B}^f_{2:T}}{\alpha_{1:T}},
%\end{equation*}
%
%where $B_{1:T} = \sum_{t=1}^T \alpha_t  B^f(x^*, x_t)$ and $\bar{B}_{2:T} = \sum_{t=1}^T \alpha_{1:t-1} B^f(\bar{x}_{t-1}, \bar{x}_t)$. In particular, $B_{1:T} \geq 0$ and by %smoothness in $f$, $\bar{B}_{2:T} \geq \sum_{t=2}^T (\alpha_{1:t-1}/2L) \| \nabla f(\bar{x}_{t-1}) - \nabla f(\bar{x}_t) \|_{*}^2$, which yields the bound
%
\begin{equation}\label{eq:ub}
\begin{aligned}
&f(\bar{x}_T) - f(x^*) \leq \\
& \frac{1}{\alpha_{1:T}} \left(R^x(T) - \frac{\alpha_{t}}{2L} \| \nabla f(\bar{x}_t) - \nabla f(x^*) \|_{*}^2 - \frac{\alpha_{1:t-1}}{2L} \| \nabla f(\bar{x}_{t-1}) - \nabla f(\bar{x}_t) \|_{*}^2
\right),
\end{aligned}
\end{equation}
were we define $R^x(T) \coloneqq \sum_{t=1}^T \alpha_t \inp{\nabla f(\bar{x}_t)}{x_t - x^*}$ as the regret of the sequence $\{x_t\}_{t=1}^T$ against the comparator $x^*$. With this machinery in place, we now turn to deriving our main results.

\section{Analysis}\label{sec:main}

Our analytical goal is to apply the online-to-batch conversion bound in \Cref{eq:ub} to the iterates $x_1, x_2, \ldots, x_T$ that \Cref{alg:model} generates. Our main challenge is that the update rule $\varphi$ prevents a straightforward application of this bound. Instead, we must upper bound the learner's regret $R^x$ by the meta-learner's regret, which is defined in terms of the iterates $w_1, w_2, \ldots, w_T$. To this end, we may decompose $R^x$ as follows:
\begin{align*}
R^x(T) &= 
\sum_{t=1}^T \alpha_t \inp{\nabla f(\bar{x}_t)}{x_t - x^*} 
= \sum_{t=1}^T \alpha_t \inp{\nabla f(\bar{x}_t)}{\varphi(\bar{x}_{t-1}, w_t) - x^*} \\
&= \sum_{t=1}^T \alpha_t \inp{\nabla f(\bar{x}_t)}{\varphi(\bar{x}_{t-1}, w_t) - \varphi(\bar{x}_{t-1}, w^*)} + \sum_{t=1}^T \alpha_t \inp{\nabla f(\bar{x}_t)}{\varphi(\bar{x}_{t-1}, w^*) - x^*}.
\end{align*}
The first term in the final expression can be understood as the regret under convex losses $\ell_t(\cdot) =  \alpha_t \inp{\nabla f(\bar{x}_t)}{\varphi(\bar{x}_{t-1}, \cdot)}$. Since $\varphi(\bar{x}_{t-1}, \cdot)$ is affine, $\ell_t$ is convex and can be upper bounded by its linearisation.  The linearisation reads $\inp{D \varphi(\bar{x}_{t-1}, w_t)^T \nabla f(\bar{x}_t)}{\cdot}$, which is identical the linear losses $\inp{\nabla h_t(w_t)}{\cdot}$ faced by the meta-learner in \Cref{eq:mg}. Hence, we may upper bound $R^x(T)$ by
\begin{align}
R^x(T)
&\leq
\sum_{t=1}^T \alpha_t \inp{D \varphi(\bar{x}_{t-1}, w_t)^T \nabla f(\bar{x}_t)}{w_t - w^*} + \sum_{t=1}^T \alpha_t \inp{\nabla f(\bar{x}_t)}{\varphi(\bar{x}_{t-1}, w^*) - x^*} \nonumber\\
&=
\sum_{t=1}^T \alpha_t \inp{\nabla h_t(w_t)}{w_t - w^*} + \sum_{t=1}^T \alpha_t \inp{\nabla f(\bar{x}_t)}{\varphi(\bar{x}_{t-1}, w^*) - x^*} \nonumber \\
&=
R^w(T) + \sum_{t=1}^T \alpha_t \inp{\nabla f(\bar{x}_t)}{\varphi(\bar{x}_{t-1}, w^*) - x^*}, \label{eq:expansion}
\end{align}
where the last identity follows by definition: $R^w(T) \coloneqq \sum_{t=1}^T \alpha_t \inp{\nabla h_t(w_t)}{w_t - w^*} $. For the last term in \Cref{eq:expansion} to be negative, so that $R^w(T) \geq R^x(T)$, we need the relative power of the comparator $w^*$ to be greater than that of the comparator $x^*$. Intuitively, the comparator $x^*$ is non-adaptive. It must make one choice $x^*$ and suffer the average loss. In contrast, the comparator $w^*$ becomes adaptive under the update rule; it can only choose one $w^*$, but on each round it plays $\varphi(\bar{x}_{t-1}, w^*)$. If $\varphi$ is sufficiently flexible, this gives the comparator $w^*$ more power than $x^*$, and hence it can force the meta-learner to suffer greater regret. When this is the case, we say that regret is \emph{preserved} when moving from $x^*$ to $w^*$.

\begin{defn}\label{defn:comp}
Given $f$, $\{\alpha_t\}_{t=1}^T$, and $\{x_t\}_{t=1}^T$, an update rule $\varphi: \cX \times \cW \to \cX$ \emph{preserves regret} if there exists a comparator $w \in \cW$ that satisfies
\begin{equation}\label{eq:comp}
\sum_{t=1}^T \alpha_t \inp{\varphi(\bar{x}_{t-1}, w)}{\nabla f(\bar{x}_t)} \leq \sum_{t=1}^T \alpha_t \inp{x^*}{\nabla f(\bar{x}_t)}.
\end{equation}
If such $w$ exists, let $w^*$ denote the comparator with smallest norm $\| w\|$.
\end{defn}

By inspecting \Cref{eq:comp}, we see that if $\varphi(\bar{x}_{t-1}, \cdot)$ can be made to negatively align with the gradient $\nabla f(\bar{x}_t)$, the update rule preserves regret. Hence, any update rule that is gradient-like in its behaviour can be made to preserve regret. However, this must not hold on every step, only sufficiently often; nor does it imply that the update rule must explicitly invoke $\nabla f$; for instance, update rules that are affine in $w$ preserve regret if the diameter of $\cW$ is sufficiently large, provided the update rule is not degenerate.

\begin{lem}\label{lem:reduc}
Given $f$, $\{\alpha_t\}_{t=1}^T$, and $\{x_t\}_{t=1}^T$, if $\varphi$ preserves regret, then
\begin{equation*}
R^x(T) = \sum_{t=1}^T \alpha_t \inp{\nabla f(\bar{x}_t)}{x_t - x^*} \leq \sum_{t=1}^T \alpha_t \inp{\nabla f(\bar{x}_t)}{\varphi(\bar{x}_{t-1}, w_t) - \varphi(\bar{x}_{t-1}, w^*)} = R^w(T).
\end{equation*}
\end{lem}
Proof: \Cref{app:proofs}. With \Cref{lem:reduc}, we can provide a convergence guarantee for meta-gradients in the convex setting. The mechanics of the proof is to use online-to-batch conversion to upper bound $f(\bar{x}_T) - f(x^*) \leq R^x(T) / \alpha_{1:T}$ and then appeal to \Cref{lem:reduc} to obtain $f(\bar{x}_T) - f(x^*) \leq R^w(T) / \alpha_{1:T}$, from which point we can plug in the FTRL regret bound.

\begin{thm}\label{thm:mg}
Let $\varphi$ preserve regret and assume \Cref{alg:model} satisfies the assumptions in \Cref{sec:setup}. Then
\begin{equation*}
\begin{aligned}
f(\bar{x}_T) - f(x^*) \leq & \frac{1}{\alpha_{1:T}} \left(\frac{\| w^* \|^2}{\beta} + \sum_{t=1}^T  \frac{\lambda\beta \alpha^2_t }{2}   \|\nabla f(\bar{x}_t) \|_{*}^2  \right.
\\
&\left. \vphantom{\sum_{t=1}^T} - \frac{\alpha_{t}}{2L} \| \nabla f(\bar{x}_t) - \nabla f(x^*) \|_{*}^2 - \frac{\alpha_{1:t-1}}{2L} \| \nabla f(\bar{x}_{t-1}) - \nabla f(\bar{x}_t) \|_{*}^2
\right).
\end{aligned}
\end{equation*}
Moreover, if $x^*$ is a global minimiser of $f$, setting $\alpha_t = 1$ and $\beta = \frac{1}{\lambda L}$ yields
\begin{equation*}
f(\bar{x}_T) - f(x^*) \leq \frac{\lambda L \operatorname{diam}(\cW)}{T}.
\end{equation*}
\end{thm}
Proof: \Cref{app:proofs}.

\section{Meta-Learning meets Optimism}\label{sec:bmg}

The reason \Cref{thm:mg} fails to achieve acceleration is because the negative terms, $-\| \nabla f (\bar{x}_{t-1}) - \nabla f(\bar{x}_t)\|^2_{*}$, do not come into play. This is because the positive term in the bound involves the norm of the gradient, rather than the norm of the difference of two gradients. The former is typically a larger quantity and hence we cannot guarantee that they vanish. To obtain acceleration, we need some form of optimism. In this section, we consider an alteration to \Cref{alg:model} that uses AO-FTRL for the meta-updates. Given some sequence of hints $\{ \tilde{g}_{t} \}_{t=1}^T$, each $\tilde{g}_t \in \rR^m$, each $w_{t+1}$ is given by
\begin{equation}\label{eq:omg}
w_{t+1} = \argmin_{w \in \cW} \left(\alpha_{t+1} \tilde{g}_{t+1} + \sum_{s=1}^t \alpha_s \inp{\nabla h_s(w_s)}{w} + \frac{1}{2 \beta_t} \| w \|^2 \right).
\end{equation}

\begin{minipage}[t]{0.45\linewidth}
\begin{algorithm}[H]
\DontPrintSemicolon
\Indp
\SetKwInOut{Input}{input}
\SetKwFor{For}{for }{}{}
\Input{Weights $\{\beta_{t}\}^T_{t=1}$}
\Input{Update rule $\varphi$}
\Input{Target oracle}
\Input{Initialisation $(x_0, w_1)$}
\For{$t = 1, 2, \ldots, T$:}{
$x_{t} = x_{t-1} + \varphi(x_{t-1}, w_t)$ \\
Query $z_t$ from target oracle\\
$d_t(\cdot) = \| z_t - x_t +  \varphi(x_t, \cdot)\|^2$\\
$w_{t+1} = w_t - \beta_t \nabla d_t(w_t)$
}
\Return{$x_T$}
\;
\;
\caption{BMG in practice.}\label{alg:bmg}
\end{algorithm}
\end{minipage}
\hfill
\begin{minipage}[t]{0.54\textwidth}
\begin{algorithm}[H]
\DontPrintSemicolon
\Indp
\SetKwInOut{Input}{input}
\SetKwFor{For}{for }{}{}
\Input{Weights $\{\alpha_{t}\}^T_{t=1}, \{\beta_{t}\}^T_{t=1}$}
\Input{Update rule $\varphi$}
\Input{Hints $\{\tilde{g}_t\}^T_{t=1}$}
\Input{Initialisation $(\bar{x}_0, w_1)$}
\For{$t = 1, 2, \ldots, T$:}{
$x_t = \varphi(\bar{x}_{t-1}, w_t)$\\
$\bar{x}_t = (1-\alpha_t/\alpha_{1:t})\bar{x}_{t-1} + (\alpha_t/\alpha_{1:t}) x_t$\\
$g_t = D \varphi(\bar{x}_{t-1}, w_t)^T \nabla f(\bar{x}_t)$ \\
$v_t = \alpha_{t+1} \tilde{g}_{t+1} + \sum_{s=1}^t \alpha_s g_s $\\
$w_{t+1}  =  \arg\min_{w \in \cW} \inp{v_t}{w} +  \frac{1}{2\beta_t}\| w \|^2$
}
\Return{$\bar{x}_T$}
\caption{Convex optimistic meta-learning.}\label{alg:oml}
\end{algorithm}
\end{minipage}

Otherwise, we proceed as in \Cref{alg:model}; for a complete description, see \Cref{alg:oml}. The AO-FTRL updates do not correspond to a standard meta-update. However, we show momentarily that optimism can be instantiated via the BMG method, detailed in \Cref{alg:bmg}. The proof for optimistic meta-gradients proceed largely as in \Cref{thm:mg}, it only differs in that we apply the AO-FTRL regret bound.

\begin{thm}\label{thm:omg}
Let $\varphi$ preserve regret and assume \Cref{alg:oml} satisfy the assumptions in \Cref{sec:setup}. Then
\begin{equation*}
\begin{aligned}
f(\bar{x}_T) - f(x^*) \leq & \frac{1}{\alpha_{1:T}} \left(\frac{\| w^* \|^2}{\beta_T} + \sum_{t=1}^T  \frac{\alpha^2_t \beta_t}{2} \|D \varphi(\bar{x}_{t-1}, w_t)^T \nabla f(\bar{x}_t) - \tilde{g}_{t} \|_{*}^2  \right.
\\
&\left. \vphantom{\sum_{t=1}^T} - \frac{\alpha_{t}}{2L} \| \nabla f(\bar{x}_t) - \nabla f(x^*) \|_{*}^2 - \frac{\alpha_{1:t-1}}{2L} \| \nabla f(\bar{x}_{t-1}) - \nabla f(\bar{x}_t) \|_{*}^2
\right).
\end{aligned}
\end{equation*}
Moreover, assume each $\tilde{g}_t$ is such that $\|D \varphi(\bar{x}_{t-1}, w_t)^T \nabla f(\bar{x}_t) - \tilde{g}_{t} \|_{*}^2 \leq q \| \nabla f(\bar{x}_{t-1}) - \nabla f(\bar{x}_t) \|_{*}^2$ for some $q > 0$. If each $\alpha_t = t$ and $\beta_t = \frac{t-1}{2t q L}$, then 
\begin{equation*}
f(\bar{x}_t) - f(x^*) \leq \frac{4 q L \operatorname{diam}(\cW)}{T^2 - 1}.
\end{equation*}
\end{thm}

\begin{proof}
The proof follows the same lines as that of \Cref{thm:mg}. The only difference is that the regret of the $\{w_t\}_{t=1}^T$ sequence can be upper bounded by $\frac{\| w^* \|^2}{\beta_T} + \frac{1}{2} \sum_{t=1}^T  \alpha^2_t \beta_t \| \nabla h_t(w_t) - \tilde{g}_t \|_{*}^2 $ instead of $\frac{\| w^* \|^2}{\beta_T} + \frac{1}{2} \sum_{t=1}^T  \alpha^2_t \beta_t \| \nabla h_t(w_t) \|_{*}^2 $, as per the AO-FTRL regret bound in \Cref{eq:ftrl-regret}. The final part follows immediately by replacing the norms and plugging in the values for $\alpha$ and $\beta$.
\end{proof}

From \Cref{thm:omg}, it is clear that if $\tilde{g}_{t}$ is a good predictor of $D \varphi(\bar{x}_{t-1}, w_t)^T \nabla f(\bar{x}_t)$, then the positive term in the summation can be cancelled by the negative term. In a classical optimisation setting, $D \varphi = I_{n}$, and hence it is easy to see that simply choosing $\tilde{g}_t$ to be the previous gradient is sufficient to achieve the cancellation \citep{Joulani:2020simpler}. Indeed, this choice gives us Nesterov's Accelerated rate \citep{Wang:2021FenchelGame}. The upshot of this is that we can specialise \Cref{alg:oml} to capture Nesterov's Accelerated method by choosing $\varphi: (x, w) \mapsto w$---as in the reduction to Heavy Ball---and setting the hints to $\tilde{g}_{t} = \nabla f(\bar{x}_{t-1})$. Hence, while the standard meta-update without optimism contains Heavy Ball as a special case, the optimistic meta-update contains Nesterov Acceleration as a special case.

In the meta-learning setting, $D \varphi$ is not an identity matrix, and hence the best targets for meta-learning are different. Naively, choosing $\tilde{g}_t = D \varphi(\bar{x}_{t-1}, w_t)^T \nabla f(\bar{x}_{t-1})$ would lead to a similar cancellation, but this is not allowed. At iteration $t$, we have not computed $w_t$ when $\tilde{g}_t$ is chosen, and hence $D \varphi(\bar{x}_{t-1}, w_t)$ is not available. The nearest term that is accessible is $D \varphi(\bar{x}_{t-2}, w_{t-1})$.

\begin{cor}\label{cor:omg}
Let each $\tilde{g}_{t+1} = D \varphi(\bar{x}_{t-1}, w_t)^T \nabla f(\bar{x}_{t})$. Assume that $\varphi$ satisfies 
\begin{equation*}
\left\| D \varphi(x', w)^T \nabla f(x) - D \varphi(x'', w')^T \nabla f(x') \right\|_{*}^2 \leq \tilde{\lambda} \left\| \nabla f(x') - \nabla f(x) \right\|_{*}^2
\end{equation*}
for all $x'', x', x \in \cX$ and $w, w' \in \cW$, for some $\tilde{\lambda} > 0$. If each $\alpha_t = t$ and $\beta_t = \frac{t-1}{2t \tilde{\lambda} L}$, then $f(\bar{x}_T) - f(x^*) \leq \frac{4 \tilde{\lambda} L \operatorname{diam}(\cW)}{T^2 - 1}$.
\end{cor}
Proof: \Cref{app:proofs}.

\subsection{Bootstrapped Meta-Gradients}

In this section, we present a simplified version of BMG for clarity, with \Cref{app:bmg} providing a fuller comparison. Essentially, BMG alters the meta-update in \Cref{alg:mg}; instead of directly minimising the loss $f$, it introduces a sequence of targets $z_1, z_2, \ldots$ and the meta-learner's goal is select $w$ so that the updated parameters minimise the distance these targets. Concretely, given an update $x_t = x_{t-1} + \varphi(x_{t-1}, w_t)$, targets are \emph{bootstrapped} from $x_t$, meaning that a vector $y_t$ is computed to produce the target $z_t = x_t - y_t$. Assuming the distance to the target is measured under $\frac{1}{2}\| \cdot \|_2^2$, the BMG meta-update takes the form
\begin{align*}
w_{t+1} &= w_t - D \varphi(x_{t-1}, w_t)^T y_t.
\end{align*}
Depending on how $y_t$ is computed, it can encode optimism. For instance, the authors rely on the update rule itself to compute a tangent $y_t =\varphi(x_t, w_t) - \nabla f(x_t + \varphi(x_t, w_t))$. This encodes optimism via $\varphi$ because it encourages the meta-learner to build up momentum (i.e. to accumulate past updates). We can contrast this with the types of updates produced by AO-FTRL in \Cref{eq:omg}. If we have hints $\tilde{g}_{t+1} = D \varphi(\bar{x}_{t-1}, w_t)^T \tilde{y}_{t+1}$ for some $\tilde{y}_{t+1} \in \rR^n$ and set $\| \cdot \| = \| \cdot \|_2$; assuming an interior solution, \Cref{eq:omg} yields

\begin{equation}\label{eq:bmg-opt}
w_{t+1} = w_{t} - \underbrace{D \varphi(\bar{x}_{t-1}, w_t)^T(\alpha_{t+1} \tilde{y}_{t+1} + \alpha_t \nabla f(\bar{x}_t))}_{\text{BMG update}} + \underbrace{\alpha_t D \varphi(\bar{x}_{t-2}, w_{t-1})^T \tilde{y}_t}_{\text{FTRL error correction}}.
\end{equation}

Hence, BMG encodes very similar dynamics to those of AO-FTRL in \Cref{eq:omg}. Under this choice of hints, the main qualitative difference is that AO-FTRL includes a correction term. The effect of this term is to ``undo'' previous hints to avoid feedback loops. Notably, BMG can suffer from divergence due to feedback if the gradient in $y_t$ is not carefully scaled \citep{flennerhag2022bootstrapped}. Our theoretical analysis suggests a simple correction method that may stabilize BMG in practice. 

More generally, targets in BMG are isomorphic to the hint function in AO-FTRL if the measure of distance in BMG is a Bregman divergence under a strongly convex function (\Cref{app:bmg}). An immediate implication of this is that the hints in \Cref{cor:omg} can be expressed as targets in BMG, and hence if BMG satisfies the assumptions involved, it converges at a rate $O(\tilde{\lambda}/T^2)$. More generally, \Cref{thm:omg} provides a sufficient condition for any target bootstrap in BMG to achieve acceleration.

\begin{cor}\label{cor:bmg}
Let each $\tilde{g}_{t+1} = D \varphi(\bar{x}_{t-1}, w_t)^T \tilde{y}_{t+1}$, for some $\tilde{y}_{t+1} \in \rR^n$. If each $\tilde{y}_{t+1}$ is a better predictor of the next gradient than $\nabla f(\bar{x}_{t-1})$, in the sense that
\begin{equation*}
\|D \varphi(\bar{x}_{t-2}, w_{t-1})^T \tilde{y}_{t} - D \varphi(\bar{x}_{t-1}, w_t)^T \nabla f(\bar{x}_t) \|_{*} \leq \tilde{\lambda} \| \nabla f(\bar{x}_t) - \nabla f(\bar{x}_{t-1}) \|_{*},
\end{equation*}
then \Cref{alg:oml} guarantees convergence at a rate $O(\tilde{\lambda}/T^2)$.
\end{cor}

\section{Conclusion}

This paper explores a connection between convex optimisation and meta-learning. We construct an algorithm for convex optimisation that aligns as closely as possible with how meta-learning is done in practice. Meta-learning introduces a transformation and we study the effect this transformation has on the rate of convergence. We find that, while a meta-learned update rule cannot generate a better dependence on the horizon $T$, it can improve upon classical optimisation up to a constant factor. 

An implication of our analysis is that for meta-learning to achieve acceleration, it is important to introduce some form of optimism. From a classical optimisation point of view, such optimism arises naturally by providing the meta-learner with hints. If hints are predictive of the learning dynamics these can lead to significant acceleration. We show that the recently proposed BMG method provides a natural avenue to incorporate optimism in practical application of meta-learning. Because targets in BMG and hints in optimistic online learning commute, our results provide first rigorous proof of convergence for BMG, while providing a general condition under which optimism in BMG yields accelerated learning. 

\clearpage

\bibliographystyle{abbrvnat}
\bibliography{refs}

\clearpage

\appendix

\section*{Appendix}

\section{Notation} \label{sec:notation_table}
\begin{table}[h]
\caption{Notation}
\label{table:notation_table}

\begin{tabularx}{\textwidth}{p{0.25\textwidth}X}
\toprule
  \multicolumn{2}{l}{{\bf Indices}}                                       \\
  $t$ & Iteration index: $t \in\{1,...,T\}$.  \\ 
  $T$ & Total number of iterations. \\
  $[T]$ & The set $\{1, 2, \ldots, T\}$.\\
  $i$ & Component index: $x^i$ is the $i$th component of $x = (x^1, \ldots, x^n)$.  \\   
  $\alpha_{a:b}$ & Sum of weights: $\alpha_{a:b} = \sum_{s=a}^b \alpha_s$\\
  $x_{a:b}$ & Weighted sum: $x_{a:b} = \sum_{s=a}^b \alpha_s x_s$\\
  $\bar{x}_{a:b}$ & Weighted average: $\bar{x}_{a:b} = x_{a:b}/ \alpha_{a:b}$\\
  \\
  \multicolumn{2}{l}{{\bf Parameters}}                                       \\  
  $x^* \in \cX$ & Minimiser of $f$.\\
  $x_t \in \cX$  &  Parameter at time $t$ \\
  $\bar x _t \in \cX$ & Moving average of $\{x_s\}_{s=1}^t$ under weights $\{\alpha_{s}\}_{s=1}^t$.\\
  $\rho_t \in (0, \infty)$ & Moving average coefficient $\alpha_t / \alpha_{1:t}$.  \\
  $w_t \in \cW$ & Meta parameters \\
  $w^* \in \cX$ & $w \in \cW$ that retains regret with smallest norm $\| w \|$.\\
  $\alpha_t  \in (0, \infty)$ & Weight coefficients \\
  $\beta_t  \in (0, \infty)$ & Meta-learning rate \\
  \\
    \multicolumn{2}{l}{{\bf Maps}}  \\  
  $f: \cX \to \rR$ & Objective function \\
  $\| \cdot \|: \cX \to \rR$ & Norm on $\cX$.\\
  $\| \cdot \|_{*}: \cX^* \to \rR$ & Dual norm of $\| \cdot \|$.\\
  $h_t: \cW \to \rR$ & Online loss faced by the meta learner\\
  $R^x(T)$ & Regret of $\{x_t\}_{t=1}^T$ against $x^*$: $R^x(T) \coloneqq \sum_{t=1}^T \alpha_t \inp{\nabla f(\bar{x}_t)}{x_t - x^*}$.\\
  $R^w(T)$ &  $R^w(T) \coloneqq \sum_{t=1}^T \alpha_t \inp{\nabla f(\bar{x}_t)}{\varphi(\bar{x}_{t-1}, w_t) - \varphi(\bar{x}_{t-1}, w^*)}$.\\
  $\varphi: \rR^n \times \rR^m \to \rR^n$ & Generic update rule used in practice\\
  $D\varphi(x, \cdot): \rR^m \! \to\! \rR^{n \times m}$ & Jacobian of $\varphi$ w.r.t. its second argument, evaluated at $x \in \rR^n$.\\
  $\varphi: \cX \times \cW \to \cX$ & Update rule in convex setting\\
  $D\varphi(x, \cdot): \cW \to \rR^{n \times m}$ & Jacobian of $\varphi$ w.r.t. its second argument, evaluated at $x \in \cX$.\\
  $B^\mu : \rR^n\! \times\! \rR^n \! \to \! [0, \infty)$ & Bregman divergence under $\mu: \rR^n \to \rR$. \\
  $\mu: \rR^n \to \rR$ & Convex distance generating function.\\
  
  \bottomrule
 \end{tabularx}
\end{table}

\clearpage

\section{Convex Quadratic Experiments}\label{app:experiments}

\paragraph{Loss function.} We consider the problem of minimising a convex quadratic loss functions $f: \rR^2 \to \rR$ of the form $f(x) = x^T Q x$, where $Q$ is randomly sampled as follows. We sample a random orthogonal matrix $U$ from the Haar distribution \texttt{scipy.stats.ortho\_group}. We construct a diagonal matrix of eigenvalues, ranked smallest to largest, with $\lambda_i = i^2$. Hence, the first dimension has an eigenvalue $1$ and the second dimension has eigenvalue $4$. The matrix $Q$ is given by $U^T \operatorname{diag}(\lambda_1, \ldots, \lambda_n) U$.

\paragraph{Protocol.} Given that the solution is always $(0, 0)$, this experiment revolves around understanding how different algorithms deal with curvature. Given symmetry in the solution and ill-conditioning, we fix the initialisation to $x_0 = (4, 4)$ for all sampled $Q$s and all algorithms and train for $100$ iterations. For each $Q$ and each algorithm, we sweep over the learning rate, decay rate, and the initialization of $w$ see \Cref{tab:hypers}. For each method, we then report the results for the combination of hyper parameters that performed the best. 

\paragraph{Results.} We report the learning curves for the best hyper-parameter choice for 5 randomly sampled problems in the top row of \Cref{fig:conv-2d} (columns correspond to different Q). We also study the sensitivity of each algorithm to the learning rate in the bottom row \Cref{fig:conv-2d}. For each learning rate, we report the cumulative loss during training. While baselines are relatively insensitive to hyper-parameter choice, meta-learned improve for certain choices, but are never worse than baselines.

\begin{table}[t]
    \centering
    \caption{Hyper-parameter sweep on Convex Quadratics. All algorithms are tuned for learning rate and initialisation of $w$. Baselines are tuned for decay rate; meta-learned variant are tuned for the meta-learning rate.}
    \begin{tabular}{l l}
        \toprule
        Learning rate & [.1, .3, .7, .9, 3., 5.] \\
        $w$ init scale & [0., 0.3, 1., 3., 10., 30.] \\
        \midrule
        Decay rate / Meta-learning rate & [0.001, 0.003, 0.01, .03, .1, .3, 1., 3., 10., 30.] \\
        \bottomrule
    \end{tabular}
    \label{tab:hypers}
\end{table}

\section{Imagenet Experiments}\label{app:imagenet}

\paragraph{Protocol.} We train a 50-layer ResNet following the Haiku example, available at \url{https://github.com/deepmind/dm-haiku/blob/main/examples/imagenet}. We modify the default setting to run with SGD. We compare default SGD to variants that meta-learn an element-wise learning rate online, i.e. $(x, w) \mapsto w \odot \nabla f(x)$. For each variant, we sweep over the learning rate (for SGD) or meta-learning rate. We report results for the best hyper-parameter over three independent runs. 

\paragraph{Standard meta-learning.} In the standard meta-learning setting, we apply the update rule once before differentiating w.r.t. the meta-parameters. That is, the meta-update takes the form $w_{t+1} = w_t - \beta \nabla h_t(w_t)$, where $h_t = f(x_t + w_t \odot \nabla f(x_t))$. Because the update rule is linear in $w$, we can compute the meta-gradient analytically: 
\begin{equation*}
\nabla h_t(w_t) = \nabla_w f(x + \varphi(x, w)) = D \varphi(x, w)^T \nabla f(x') = \nabla f(x) \odot \nabla f(x'),
\end{equation*}
where $x' = x + \varphi(x, w)$. Hence, we can compute the meta-updates in \Cref{alg:mg} manually as $w_{t+1} = \max\{w_t - \beta \nabla f(x_t) \odot \nabla f(x_{t+1}), 0.\}$, where we introduce the $\max$ operator on an element-wise basis to avoid negative learning rates. Empirically, this was important to stabilize training.  

\paragraph{Optimistic meta-learning.} For optimistic meta-learning, we proceed much in the same way, but include a gradient prediction $\tilde{g}_{t+1}$. For our prediction, we use the previous gradient, $\nabla f(x_{t+1})$, as our prediction. Following \Cref{eq:bmg-opt}, this yields meta-updates of the form
\begin{equation*}
w_{t+1} = \max \, \Big\{w_t - \beta \nabla f(x_{t+1}) \odot \left(\nabla f(x_{t+1}) + \nabla f(x_{t})\right) - \nabla f(x_t) \odot \nabla f(x_t), 0.\Big\}.
\end{equation*}

\paragraph{Results.} We report Top-1 accuracy on the held-out test set as a function of training steps in \Cref{fig:imagenet}. Tuning the learning rate does not yield any statistically significant improvements under standard meta-learning. However, with optimistic meta-learning, we obtain a significant acceleration as well as improved final performance, increasing the mean final top-1 accuracy from $~72\%$ to $~75\%$.

\begin{table}[t]
    \centering
    \caption{Hyper-parameter sweep on Imagenet.}
    \begin{tabular}{l l}
        \toprule
        (Meta-)learning rate & [0.001, 0.01, 0.02, 0.05, 0.1] \\
        \bottomrule
    \end{tabular}
    \label{tab:hypers:imagenet}
\end{table}

\section{Proofs}\label{app:proofs}

This section provides complete proofs. We restate the results for convenience.

\textbf{\Cref{lem:reduc}.} \emph{Given $f$, $\{\alpha_t\}_{t=1}^T$, and $\{x_t\}_{t=1}^T$, if $\varphi$ preserves regret, then}
\begin{equation*}
R^x(T) = \sum_{t=1}^T \alpha_t \inp{\nabla f(\bar{x}_t)}{x_t - x^*} \leq \sum_{t=1}^T \alpha_t \inp{\nabla f(\bar{x}_t)}{\varphi(\bar{x}_{t-1}, w_t) - \varphi(\bar{x}_{t-1}, w^*)} = R^w(T).
\end{equation*}
\begin{proof}
Starting from $R^x$ in \Cref{eq:expansion}, if the update rule preserves regret, there exists $w^* \in \cW$ for which
\begin{align*}
R^x(T) =&\sum_{t=1}^T \alpha_t \inp{\nabla f(\bar{x}_T)}{\varphi(\bar{x}_{t-1}, w_t) - x^*}\\
&=
\sum_{t=1}^T \alpha_t \inp{\nabla f(\bar{x}_T)}{\varphi(\bar{x}_{t-1}, w_t) - \varphi(\bar{x}_{t-1}, w^*)} +
\sum_{t=1}^T \alpha_t \inp{\nabla f(\bar{x}_T)}{\varphi(\bar{x}_{t-1}, w^*) - x^*} \\
&\leq 
\sum_{t=1}^T \alpha_t \inp{\nabla f(\bar{x}_T)}{\varphi(\bar{x}_{t-1}, w_t) - \varphi(\bar{x}_{t-1}, w^*)} = R^w(T),
\end{align*}
since $w^*$ is such that $\sum_{t=1}^T \alpha_t \inp{\nabla f(\bar{x}_T)}{\varphi(\bar{x}_{t-1}, w^*) - x^*} \leq 0$.
\end{proof}

\textbf{\Cref{thm:mg}.} \emph{Let $\varphi$ preserve regret and assume \Cref{alg:model} satisfy the assumptions in \Cref{sec:setup}. Then}
\begin{equation*}
\begin{aligned}
f(\bar{x}_T) - f(x^*) \leq & \frac{1}{\alpha_{1:T}} \left(\frac{\| w^* \|^2}{\beta} + \sum_{t=1}^T  \frac{\lambda\beta \alpha^2_t }{2}   \|\nabla f(\bar{x}_t) \|_{*}^2  \right.
\\
&\left. \vphantom{\sum_{t=1}^T} - \frac{\alpha_{t}}{2L} \| \nabla f(\bar{x}_t) - \nabla f(x^*) \|_{*}^2 - \frac{\alpha_{1:t-1}}{2L} \| \nabla f(\bar{x}_{t-1}) - \nabla f(\bar{x}_t) \|_{*}^2
\right).
\end{aligned}
\end{equation*}
\emph{If $x^*$ is a global minimiser of $f$, setting $\alpha_t = 1$ and $\beta = \frac{1}{\lambda L}$ yields $f(\bar{x}_T) - f(x^*) \leq \frac{\lambda L \operatorname{diam}(\cW)}{T}$.}
\begin{proof}
Since $\varphi$ preserves regret, by \Cref{lem:reduc}, the regret term $R^x(T)$ in \Cref{eq:ub} is upper bounded by $R^w(T)$. We therefore have
\begin{equation}\label{eq:ub2}
\begin{aligned}
&f(\bar{x}_T) - f(x^*) \leq \\
& \frac{1}{\alpha_{1:T}} \left(R^w(T) - \frac{\alpha_{t}}{2L} \| \nabla f(\bar{x}_t) - \nabla f(x^*) \|_{*}^2 - \frac{\alpha_{1:t-1}}{2L} \| \nabla f(\bar{x}_{t-1}) - \nabla f(\bar{x}_t) \|_{*}^2
\right).
\end{aligned}
\end{equation}
Next, we need to upper-bound $R^w(T)$. Since, $R^w(T) = \sum_{t=1}^T \alpha_t \inp{\nabla f(\bar{x}_T)}{\varphi(\bar{x}_{t-1}, w_t) - \varphi(\bar{x}_{t-1}, w^*)} $, the regret of $\{w_t\}_{t=1}^T$ is defined under loss functions $h_t: \cW \to \rR$ given by $h_t =\alpha_t \inp{\nabla f(\bar{x}_T)}{\varphi(\bar{x}_{t-1}, w))}$. By assumption of convexity in $\varphi$, each $h_t$ is convex in $w$. Hence, the regret under $\{\alpha_t h_t\}_{t=1}^T$ can be upper bounded by the regret under the linear losses $\{\alpha_t \inp{\nabla h_t(w_t)}{\cdot}\}_{t=1}^T$. These linear losses correspond to the losses used in the meta-update in \Cref{eq:mg}. Since the meta-update is an instance of FTRL, we may upper-bound $R^w(T)$ by \Cref{eq:ftrl-regret} with each $\tilde{g}_t = 0$. Putting this together along with smoothness of $\varphi$, 
\begin{align}
R^x(T) &\leq R^w(T) \nonumber\\
&= \sum_{t=1}^T \alpha_t \inp{\nabla f(\bar{x}_T)}{\varphi(\bar{x}_{t-1}, w_t) - \varphi(\bar{x}_{t-1}, w^*)} \nonumber\\
&\leq \sum_{t=1}^T \alpha_t \inp{\nabla h_t(w_t)}{w_t - w^*} \nonumber\\
&\leq \frac{\| w^* \|^2}{\beta} + \frac{ \beta}{2} \sum_{t=1}^T  \alpha^2_t \| \nabla h_t(w_t) \|_{*}^2 \nonumber\\
&= \frac{\| w^* \|^2}{\beta} + \frac{ \beta}{2} \sum_{t=1}^T  \alpha^2_t \| D \varphi(\bar{x}_{t-1}, w_t)^T \nabla f(\bar{x}_t) \|_{*}^2 \nonumber\\
&\leq \frac{\| w^* \|^2}{\beta} + \frac{\lambda \beta}{2} \sum_{t=1}^T  \alpha^2_t \|\nabla f(\bar{x}_t) \|_{*}^2. \label{eq:ub3}
\end{align}
Putting \Cref{eq:ub2} and \Cref{eq:ub3} together gives the stated bound. Next, if $x^*$ is the global optimiser, $\nabla f(x^*) = 0$ by first-order condition. Setting $\beta=1/(L\lambda)$ and $\alpha_t = 1$ means the first two norm terms in the summation cancel. The final norm term in the summation is negative and can be ignored. We are left with $f(\bar{x}_T) - f(x^*) \leq \frac{\lambda L \| w^*\|^2}{T} \leq \frac{\lambda L \operatorname{diam}(\cW)}{T}$.
\end{proof}

\textbf{\Cref{cor:omg}.} \emph{Let each $\tilde{g}_{t+1} = D \varphi(\bar{x}_{t-1}, w_t)^T \nabla f(\bar{x}_{t})$. Assume that $\varphi$ satisfies}
\begin{equation*}
\left\| D \varphi(x', w)^T \nabla f(x) - D \varphi(x'', w')^T \nabla f(x') \right\|_{*}^2 \leq \tilde{\lambda} \left\| \nabla f(x') - \nabla f(x) \right\|_{*}^2
\end{equation*}
\emph{for all $x'', x', x \in \cX$ and $w, w' \in \cW$, for some $\tilde{\lambda} > 0$. If each $\alpha_t = t$ and $\beta_t = \frac{t-1}{2t \tilde{\lambda} L}$, then $f(\bar{x}_T) - f(x^*) \leq \frac{4 \tilde{\lambda} L \operatorname{diam}(\cW)}{T^2 - 1}$.}
\begin{proof}
Plugging in the choice of $\tilde{g}_t$ and using that 
\begin{equation*}
\left\| D \varphi(\bar{x}_{t-1}, w_t)^T \nabla f(\bar{x}_t) - D \varphi(x_{t-2}, w_{t-1})^T \nabla f(\bar{x}_{t-1}) \right\|_{*}^2 \leq \tilde{\lambda} \left\| \nabla f(\bar{x}_{t-1}) - \nabla f(\bar{x}_t) \right\|_{*}^2,
\end{equation*}
the bound in \Cref{thm:omg} becomes
\begin{equation*}
\begin{aligned}
f(\bar{x}_T) - f(x^*) \leq & \frac{1}{\alpha_{1:T}} \left(\frac{\| w^* \|^2}{\beta_T} + \frac{1}{2}\sum_{t=1}^T  \left(\tilde{\lambda}\alpha^2_t \beta_t  - \frac{\alpha_{1:t-1}}{L}\right) \|\nabla f(\bar{x}_t) - \nabla f(\bar{x}_{t-1}) \|_{*}^2\right),
\end{aligned}
\end{equation*}
where we drop the negative terms $\| \nabla f(\bar{x}_t) - \nabla f(x^*) \|_{*}^2$. Setting $\alpha_t = t$ yields $\alpha_{1:t-1} = \frac{(t-1) t}{2}$, while setting $\beta_t = \frac{t-1}{2t \tilde{\lambda} L}$ means $\tilde{\lambda}\alpha_t^2 \beta_t = \frac{(t-1)t}{2 L}$. Hence, $\tilde{\lambda}\alpha_t^2 \beta_t - \alpha_{1:t-1}/L$ cancels and we get 
\begin{equation*}
f(\bar{x}_T) - f(x^*) \leq \frac{\| w^* \|^2}{\beta_T \alpha_{1:T}} = \frac{4 \| w^* \|^2 \tilde{\lambda} L }{(T-1) (T+1)} \leq \frac{4\tilde{\lambda} L \operatorname{diam}(\cW)}{(T-1)(T+1)} = \frac{4\tilde{\lambda} L \operatorname{diam}(\cW)}{T^2-1}.
\end{equation*}
\end{proof}

\textbf{\Cref{cor:bmg}.} \emph{Let each $\tilde{g}_{t+1} = D \varphi(\bar{x}_{t-1}, w_t)^T \tilde{y}_{t+1}$, for some $\tilde{y}_{t+1} \in \rR^n$. If each $\tilde{y}_{t+1}$ is a better predictor of the next gradient than $\nabla f(\bar{x}_{t-1})$, in the sense that}
\begin{equation*}
\|D \varphi(\bar{x}_{t-2}, w_{t-1})^T \tilde{y}_{t} - D \varphi(\bar{x}_{t-1}, w_t)^T \nabla f(\bar{x}_t) \|_{*} \leq \tilde{\lambda} \| \nabla f(\bar{x}_t) - \nabla f(\bar{x}_{t-1}) \|_{*},
\end{equation*}
\emph{then \Cref{alg:oml} guarantees convergence at a rate $O(\tilde{\lambda}/T^2)$.}
\begin{proof}
The proof follows the same argument as \Cref{cor:omg}.
\end{proof}

\section{BMG}\label{app:bmg}

\emph{Errata: this was incorrectly referred to as Appendix F in our original submission.}

In this section, we provide a more comprehensive reduction of BMG to AO-FTRL. First, we provide a more general definition of BMG. Let $\mu: \cX \to \rR$ be a convex distance generating function and define the Bregman Divergence $B^\mu: \rR^n \times \rR^n \to \rR$ by
\begin{equation*}
B^\mu_z(x) = \mu(x) - \mu(z) - \inp{\nabla \mu(z)}{x - z}.
\end{equation*}
Given initial condition $(x_0, w_1)$, the BMG updates proceed according to
\begin{align}
x_t = x_{t-1} + \varphi(x_{t-1}, w_t) \nonumber\\
w_{t+1} = w_{t} - \beta_t \nabla d_t(w_t),\label{eq:bmg-meta-update}
\end{align}
where $d_t: \rR^n \to \rR$ is defined by $d_t(w) = B^\mu_{z_t}(x_{t-1} + \varphi(x_{t-1}, w_t))$, where each $z_t \in \rR^n$ is referred to as a target. See \Cref{alg:bmg2} for an algorithmic summary. A bootstrapped target uses the meta-learner's most recent update, $x_t$, to compute the target, $z_t = x_t + y_t$ for some tangent vector $y_t \in \rR^n$. This tangent vector represents a form of optimism, and provides a signal to the meta-learner as to what would have been a more efficient update. In particular, the author's consider using the meta-learned update rule to construct $y_t$; $y_t = \varphi(x_{t}, w_t) - \nabla f(x_t \varphi(x_t, w-t))$. Note that $x_t = x_{t-1} + \varphi(x_{t-1}, w_t)$, and hence this tangent vector is obtained by applying the update rule again, but now to $x_t$. For this tangent to represent an improvement, it must be \emph{assumed that $w_t$ is a good parameterisation}. Hence, bootstrapping represents a form of optimism. To see how BMG relates to \Cref{alg:oml}, and in particular, \cref{eq:omg}, expand \Cref{eq:bmg-meta-update} to get
\begin{equation}
w_{t+1} = w_t - \beta_t D \varphi(x_{t-1}, w_t)^T \left(\nabla \mu(x_{t}) - \nabla \mu(z_t) \right).
\end{equation}

\begin{algorithm}[t!]
\DontPrintSemicolon
\Indp
\SetKwInOut{Input}{input}
\SetKwFor{For}{for }{}{}
\Input{Weights $\{\rho_{t}\}^T_{t=1}, \{\beta_{t}\}^T_{t=1}$}
\Input{Update rule $\varphi$}
\Input{Matching function $B^\mu$}
\Input{Target oracle}
\Input{Initialisation $(x_0, w_1)$}
\For{$t = 1, 2, \ldots, T$:}{
$x_{t} = x_{t-1} + \varphi(x_{t-1}, w_t)$ \\
Query $z_t$ from target oracle\\
$d_t: w \mapsto B^\mu_{z_t}(x_{t-1} + \varphi(x_{t-1}, w))$\\
$w_{t+1} = w_t - \beta_t \nabla d_t(w_t)$
}
\Return{$x_T$}
\;
\;
\caption{BMG in practice (general version).}\label{alg:bmg2}
\end{algorithm}

In contrast, AO-FTRL reduces to a slightly different type of update.

\begin{lem}\label{lem:oftrl-reduc}
Consider \Cref{alg:oml}. Given online losses $h_t: \cW \to \rR$ defined by $\{\inp{D \varphi(\bar{x}_{t-1}, w_t)^T \nabla f(\bar{x}_t)}{\cdot}\}_{t=1}^T$ and hint functions $\{\inp{\tilde{g}_{t}, \cdot}\}_{t=1}^T$, with each $\tilde{g}_{t} \in \rR^m$. If $\| \cdot \| = (1/2)\| \cdot \|_2$, an interior solution to \Cref{eq:omg} is given by
\begin{equation*}
w_{t+1} = \frac{\beta_t}{\beta_{t-1}} w_t - \beta_t \left(\alpha_{t+1} \tilde{g}_{t+1} + \alpha_{t} (D \varphi(\bar{x}_{t-1}, w_t)^T \nabla f(\bar{x}_t) - \tilde{g}_t) \right).
\end{equation*}
\end{lem}

\begin{proof}
By direct computation:
\begin{align*}
w_{t+1} 
&= \argmin_{w \in \cW} \left(\alpha_{t+1} \inp{\tilde{g}_{t+1}}{w} + \sum_{s=1}^t \alpha_s \inp{D \varphi(\bar{x}_{s-1}, w_s)^T \nabla f(\bar{x}_s)}{w} + \frac{1}{2\beta_{t}}\|w\|_2^2\right) \\
&= -\beta_t \left( \alpha_{t+1} \tilde{g}_{t+1} + \sum_{s=1}^t \alpha_t  D \varphi(\bar{x}_{s-1}, w_s)^T \nabla f(\bar{x}_s))\right) \\
&= -\beta_t \left( \alpha_{t+1} \tilde{g}_{t+1} + \alpha_{t} D \varphi(\bar{x}_{t-1}, w_t)^T \nabla f(\bar{x}_t) + \left( \sum_{s=1}^{t-1} \alpha_t  D \varphi(\bar{x}_{s-1}, w_s)^T \nabla f(\bar{x}_s)) \right)\right)\\
&= -\beta_t \left( \alpha_{t+1} \tilde{g}_{t+1} + \alpha_{t} (D \varphi(\bar{x}_{t-1}, w_t)^T \nabla f(\bar{x}_t) - \tilde{g}_t)\right) \\ 
&\quad -\beta_t \left(\alpha_t \tilde{g}_{t} + \sum_{s=1}^{t-1} \alpha_t  D \varphi(\bar{x}_{s-1}, w_s)^T \nabla f(\bar{x}_s)) \right)\\
&= \frac{\beta_t}{\beta_{t-1}} w_t - \beta_t \left(\alpha_{t+1} \tilde{g}_{t+1} + \alpha_{t} (D \varphi(\bar{x}_{t-1}, w_t)^T \nabla f(\bar{x}_t) - \tilde{g}_t) \right).
\end{align*}
\end{proof}

AO-FTRL includes a decay rate $\beta_t / \beta_{t-1}$; this decay rate can be removed by instead using optimistic online mirror descent \citep{Rakhlin:2013optimism,Joulani:2017modular}---to simplify the exposition we consider only FTRL-based algorithms in this paper. An immediate implication of \Cref{lem:oftrl-reduc} is the error-corrected version of BMG.

\begin{cor}\label{cor:oftrl-reduc-bmg-corr}
Setting $\tilde{g}_{t+1} = D \varphi(\bar{x}_{t-1}, w_t)^T \tilde{g}_{t+1}$ for some $\tilde{y}_{t+1} \in \rR^n$ yields an error-corrected version of the BMG meta-update in \Cref{eq:bmg-meta-update}. Specifically, the meta-updates in \Cref{lem:oftrl-reduc} becomes
\begin{equation*}
w_{t+1} = \frac{\beta_{t}}{\beta_{t-1}} w_{t} - \underbrace{\beta_{t} D \varphi(\bar{x}_{t-1}, w_t)^T(\alpha_{t+1} \tilde{y}_{t+1} + \alpha_t \nabla f(\bar{x}_t))}_{\text{BML update}} + \underbrace{\beta_t \alpha_t D \varphi(\bar{x}_{t-2}, w_{t-1})^T \tilde{y}_t}_{\text{FTRL error correction}}.
\end{equation*}
\end{cor}

\begin{proof}
Follows immediately by substituting for each $\tilde{g}_{t+1}$ in \Cref{lem:oftrl-reduc}.
\end{proof}

To illustrate this connection, Let $\mu = f$. In this case, the BMG update reads $w_{t+1} = w_{t} - \beta_t D \varphi({x}_{t-1}, w_t)^T (\nabla f(z_{t}) - \nabla f(x_t))$. The equivalent update in the convex optimisation setting (i.e. \Cref{alg:oml}) is obtained by setting $\tilde{y}_{t+1} = \nabla f (z_t)$, in which case \Cref{cor:oftrl-reduc-bmg-corr} yields
\begin{equation*}
w_{t+1} = \frac{\beta_{t+1}}{\beta_t} w_t - \beta_t D \varphi(\bar{x}_{t-1}, w_t)^T(\alpha_{t+1} \nabla f(z_t) - \alpha_t \nabla f(\bar{x}_t)) + \xi_t,
\end{equation*}
where $\xi_t = \beta_t \alpha_t D \varphi(\bar{x}_{t-2}, w_{t-1})^T \nabla f(\bar{x}_t-1)$ denotes the error correction term we pick up through AO-FTRL. Since \Cref{alg:bmg2} does not average its iterates---while \Cref{alg:oml} does---we see that these updates (ignoring $\xi_t$) are identical up to scalar coefficients (that can be controlled for by scaling each $\beta_t$ and each $\tilde{g}_{t+1}$ accordingly). 

More generally, the mapping from targets in BMG and hints in AO-FTRL takes on a more complicated pattern. Our next results show that we can always map one into the other. To show this, we need to assume a certain recursion. It is important to notice however that at each iteration introduces an unconstrained variable and hence the assumption on the recursion is without loss of generality (as the free variable can override it).

\begin{thm}\label{thm:bmg-iso}
Targets in \Cref{alg:bmg2} and hints in \cref{alg:oml} commute in the following sense. \textbf{BMG $\rightarrow$ AO-FTRL.} Let BMG targets $\{z_t\}_{t=1}^T$ by given. A sequence of hints $\{\tilde{g}\}_{t=1}^T$ can be constructed recursively by
\begin{equation}\label{eq:bmg-aoftrl}
\alpha_{t+1}\tilde{g}_{t+1} = D \varphi(\bar{x}_{t-1}, w_t)^T (\nabla \mu(\bar{x}_t) - \nabla \mu(z_t) - \alpha_t \nabla f(\bar{x}_t)) + \alpha_t \tilde{g}_t, \qquad t \in [T],
\end{equation}
so that interior updates for \Cref{alg:oml} are given by
\begin{equation*}
w_{t+1} = \frac{\beta_{t}}{\beta_{t-1}}w_t - \beta_t \left(\nabla \mu(z_t) - \nabla \mu(\bar{x}_t)\right).
\end{equation*}
\textbf{AO-FTRL $\rightarrow$ BMG.} Conversely, assume a sequence $\{\tilde{y}_{t}\}_{t=1}^T$ are given, each $\tilde{y}_t \in \rR^n$. If $\mu$ strictly convex, a sequence of BMG targets $\{z_t\}_{t=1}^T$ can be constructed recursively by
\begin{equation*}
z_t = \nabla \mu^{-1}\left(\nabla \mu(x_t) - (\alpha_{t+1} \tilde{y}_{t+1} + \alpha_t \nabla f(x_t))   \right) \qquad t \in [T],
\end{equation*}
so that BMG updates in \Cref{eq:bmg-meta-update} are given by
\begin{equation*}
w_{t+1} = 
w_t - \beta_t \left(\alpha_{t+1} \tilde{g}_{t+1} + \alpha_{t} (D \varphi(\bar{x}_{t-1}, w_t)^T \nabla f(\bar{x}_t) - \tilde{g}_t) \right),
\end{equation*}
where each $\tilde{g}_{t+1}$ is the BMG-induced hint function, given by
\begin{equation*}
\alpha_{t+1}\tilde{g}_{t+1} = \alpha_{t+1} D \varphi(x_{t-1}, w_t)^T \tilde{y}_{t+1} + \alpha_t \tilde{g}_t.
\end{equation*}
\end{thm}

\begin{proof}
First, consider BMG $\rightarrow$ AO-FTRL. First note that $\tilde{g}_1$ is never used and can thus be chosen arbitrarily---here, we set $\tilde{g}_1 = 0$.  For $w_2$, \Cref{lem:oftrl-reduc} therefore gives the interior update 
\begin{equation*}
w_2 = \frac{\beta_2}{\beta_1} w_1 - \beta_1 (\alpha_2 \tilde{g}_2 + \alpha_1 D \varphi(\bar{x}_0, w_1)^T \nabla f(\bar{x}_1)).  
\end{equation*}
Since the formulate for $\tilde{g}_2$ in \Cref{eq:bmg-aoftrl} only depends on quantities with iteration index $t=0,1$, we may set $\alpha_{2}\tilde{g}_t = D \varphi(\bar{x}_0, w_1)^T(\nabla \mu(\bar{x}_1) - \nabla \mu (z_t) - \alpha_t \nabla f(\bar{x}_1))$. This gives the update
\begin{equation*}
w_2 = \frac{\beta_2}{\beta_1} w_1 - \beta_1 D \varphi(\bar{x}_0, w_1)^T(\nabla \mu(\bar{x}_1) - \nabla \mu (z_1)). 
\end{equation*}
Now assume the recursion holds up to time $t$. As before, we may choose $\alpha_{t+1} \tilde{g}_{t+1}$ according to the formula in \Cref{eq:bmg-aoftrl} since all quantities on the right-hand side depend on quantities computed at iteration $t$ or $t-1$. Subtituting this into \Cref{lem:oftrl-reduc}, we have
\begin{align*}
w_{t+1} &= \frac{\beta_t}{\beta_{t-1}} w_t - \beta_t \left(\alpha_{t+1} \tilde{g}_{t+1} + \alpha_{t} (D \varphi(\bar{x}_{t-1}, w_t)^T \nabla f(\bar{x}_t) - \tilde{g}_t) \right) \\
&=
\frac{\beta_t}{\beta_{t-1}} w_t - \beta_t \left(D \varphi(\bar{x}_{t-1}, w_t)^T (\nabla \mu(\bar{x}_t) - \nabla \mu(z_t) - \alpha_t \nabla f(\bar{x}_t)) + \alpha_t \tilde{g}_t\right. \\
&\quad \left. +\alpha_{t} (D \varphi(\bar{x}_{t-1}, w_t)^T \nabla f(\bar{x}_t) - \tilde{g}_t) \right) \\
&=
\frac{\beta_t}{\beta_{t-1}} w_t - \beta_t D \varphi(\bar{x}_{t-1}, w_t)^T (\nabla \mu(\bar{x}_t) - \nabla \mu(z_t)).
\end{align*}
AO-FTRL $\rightarrow$ BMG. The proof in the other direction follows similarly. First, note that for $\mu$ strictly convex, $\nabla \mu$ is invertible. Then, $z_1 = \nabla \mu^{-1}(\nabla \mu(x_1)- (\alpha_{2} \tilde{y}_2 + \alpha_1 \nabla f(x_1)))$. This target is permissible since $x_1$ is already computed and $\{\tilde{y}_t\}_{t=1}^T$ is given. Substituting this into the BMG meta-update in \Cref{eq:bmg-meta-update}, we find
\begin{align*}
w_{2} 
&= w_1 - \beta_1 D \varphi(x_0, w_1)^T(\nabla \mu(x_1) - \nabla \mu(\nabla \mu^{-1}(\nabla \mu(x_1)- (\alpha_{2} \tilde{y}_2 + \alpha_1 \nabla f(x_1)))) )\\
&= w_1 - \beta_1 D \varphi(x_0, w_1)^T (\alpha_2 \tilde{y}_2 + \alpha_1 \nabla f(x_1))\\
&=
w_1 - \beta_1 \left(\alpha_{2} \tilde{g}_{2} + \alpha_{1} (D \varphi(\bar{x}_{0}, w_1)^T \nabla f(\bar{x}_1) - \tilde{g}_1) \right),
\end{align*}
where the last line uses that $\tilde{g}_2$ is defined by $\alpha_{2}\tilde{g}_2  - \alpha_1 \tilde{g}_1 = D \varphi(\bar{x}_{0}, w_1)^T \tilde{y}_2$ and $\tilde{g}_1$ is arbitrary. Again, assume the recursion holds to time $t$. We then have 
\begin{align*}
w_{t+1} 
&= w_t - \beta_t D \varphi(x_{t-1}, w_t)^T \left(\nabla \mu (x_t) - \nabla \mu(z_t) \right)\\
&= w_t - \beta_t D \varphi(x_{t-1}, w_t)^T(\nabla \mu(x_t) \\
& \quad
- \nabla \mu(\nabla \mu^{-1}(\nabla \mu(x_t)- (\alpha_{t+1} \tilde{y}_{t+1} + \alpha_t \nabla f(x_t)))))\\
&= w_t - \beta_t D \varphi(x_{t-1}, w_t)^T(\alpha_{t+1} \tilde{y}_{t+1} + \alpha_t \nabla f(x_t)) \\
&= w_t - \beta_t (\alpha_{t+1} \tilde{g}_{t+1} + \alpha_t (D \varphi(x_{t-1}, w_t)^T \nabla f(x_t) - \tilde{g}_t)).
\end{align*}
\end{proof}

\end{document}